\documentclass[runningheads]{llncs}
\usepackage[T1]{fontenc}
\usepackage{hyperref}
\usepackage{amsfonts}
\usepackage{enumitem}
\usepackage[english]{babel}
\usepackage{multirow,boldline}
\usepackage{textcomp}
\usepackage[dvipsnames,table]{xcolor}
\usepackage{float}
\usepackage{dsfont}
\usepackage{placeins}
\usepackage{stfloats}
\usepackage{derivative}

\usepackage{amsmath}
\usepackage{amssymb}
\usepackage{mathtools}

\usepackage{todonotes}

\begin{document}
\title{Dice Semimetric Losses:\\Optimizing the Dice Score with Soft Labels}
%
%
\author{Zifu Wang\inst{1} \and 
Teodora Popordanoska\inst{1} \and 
Jeroen Bertels\inst{1} \and 
Robin Lemmens\inst{2,3} \and 
Matthew B. Blaschko\inst{1}} 

\authorrunning{Wang et al.}
%
\institute{ESAT-PSI, KU Leuven, Leuven, Belgium \and
Department of Neurosciences, KU Leuven, Leuven, Belgium \and
Department of Neurology, UZ Leuven, Leuven, Belgium \\
\email{firstname.lastname@kuleuven.be}}

\maketitle              

\begin{abstract}
The soft Dice loss (SDL) has taken a pivotal role in numerous automated segmentation pipelines in the medical imaging community. Over the last years, some reasons behind its superior functioning have been uncovered and further optimizations have been explored. However, there is currently no implementation that supports its direct utilization in scenarios involving soft labels. Hence, a synergy between the use of SDL and research leveraging the use of soft labels, also in the context of model calibration, is still missing. In this work, we introduce Dice semimetric losses (DMLs), which (i) are by design identical to SDL in a standard setting with hard labels, but (ii) can be employed in settings with soft labels. Our experiments on the public QUBIQ, LiTS and KiTS benchmarks confirm the potential synergy of DMLs with soft labels (e.g.\ averaging, label smoothing, and knowledge distillation) over hard labels (e.g.\ majority voting and random selection). As a result, we obtain superior Dice scores and model calibration, which supports the wider adoption of DMLs in practice. The code is available at \href{https://github.com/zifuwanggg/JDTLosses}{https://github.com/zifuwanggg/JDTLosses}.
\end{abstract}

\section{Introduction}\label{sec:introduction}
Image segmentation is a fundamental task in medical image analysis. One of the key design choices in many segmentation pipelines that are based on neural networks lies in the selection of the loss function. In fact, the choice of loss function goes hand in hand with the metrics chosen to assess the quality of the predicted segmentation~\cite{NatureVapnik1995}. The intersection-over-union (IoU) and the Dice score are commonly used metrics because they reflect both size and localization agreement, and they are more in line with perceptual quality compared to, e.g., pixel-wise accuracy~\cite{OptimizationTMI2020,MetricsMaier-HeinarXiv2023}. Consequently, directly optimizing the IoU or the Dice score using differentiable surrogates as (a part of) the loss function has become prevalent in semantic segmentation~\cite{Lovasz-softmaxLossBermanCVPR2018,nnU-NetIsenseeNatureMethods2021,OptimizationTMI2020,SAMKirillovICCV2023,JMLsWangNeurIPS2023}. In medical imaging in particular, the Dice score and the soft Dice loss (SDL)~\cite{V-NetMilletari3DV2016,SoftDiceLossSudreMICCAIWorkshop2017} have become the standard practice, and some reasons behind its superior functioning have been uncovered and further optimizations have been explored~\cite{OptimizationTMI2020,TheoreticalBertelsMIA2021,TheDiceLossTilborghsMICCAI2022}.

Another mechanism to further improve the predicted segmentation that has gained significant interest in recent years, is the use of soft labels during training. Soft labels can be the result of data augmentation techniques such as label smoothing (LS)~\cite{InceptionV2V3SzegedyCVPR2016,SVLSIslamIPMI2021} and are integral to regularization methods such as knowledge distillation (KD)~\cite{KDHintonNeurIPSWorkshop2015,EMKDQinTMI2021}. Their role is to provide additional regularization so as to make the model less prone to overfitting \cite{InceptionV2V3SzegedyCVPR2016,KDHintonNeurIPSWorkshop2015} and to combat overconfidence~\cite{CalibrationGuoICML2017}, e.g., providing superior model calibration~\cite{WhenMullerNeurIPS2019}. In medical imaging, soft labels emerge not only from LS or KD, but are also present inherently due to considerable intra- and inter-rater variability. For example, multiple annotators often disagree on organ and lesion boundaries, and one can average their annotations to obtain soft label maps~\cite{SoftSegGrosMIA2021,MRNetJiCVPR2021,UsingSilvaMICCAIWorkshop2021,LabelLemayMELBA2023}.

This work investigates how the medical imaging community can combine the use of SDL with soft labels to reach a state of synergy. While the original SDL surrogate was posed as a relaxed form of the Dice score, naively inputting soft labels to SDL is possible (e.g.\ in open-source segmentation libraries~\cite{SMPIakubovskii2019,MMSegmentation2020,nnU-NetIsenseeNatureMethods2021,MedISegZhangarXiv2022}), but it tends to push predictions towards 0-1 outputs rather than make them resemble the soft labels~\cite{TheoreticalBertelsMIA2021,NoisyNordstromarXiv2023,JMLsWangNeurIPS2023}. Consequently, the use of SDL when dealing with soft labels might not align with a user's expectations, with potential adverse effects on the Dice score, model calibration and volume estimation~\cite{TheoreticalBertelsMIA2021}. 

Motivated by this observation, we first (in Sect.~\ref{sec:methods}) propose two probabilistic extensions of SDL, namely, Dice semimetric losses (DMLs). These losses satisfy the conditions of a semimetric and are fully compatible with soft labels. In a standard setting with hard labels, DMLs are identical to SDL and can safely replace SDL in existing implementations. Secondly (in Sect.~\ref{sec:experiments}), we perform extensive experiments on the public QUBIQ, LiTS and KiTS benchmarks to empirically confirm the potential synergy of DMLs with soft labels (e.g.\ averaging, LS, KD) over hard labels (e.g.\ majority voting, random selection).

\section{Methods}\label{sec:methods}
We adopt the notation from \cite{JMLsWangNeurIPS2023}. In particular, we denote the predicted segmentation as $\dot x \in \{1,...,C\}^p$ and the ground-truth segmentation as $\dot y \in \{1,...,C\}^p$, where $C$ is the number of classes and $p$ the number of pixels. For a class $c$, we define the set of predictions as $x^c=\{\dot x=c\}$, the set of ground-truth as $y^c=\{\dot y=c\}$, the union as $u^c=x^c\cup y^c$, the intersection as $v^c=x^c\cap y^c$, the symmetric difference (i.e., the set of mispredictions) as $m^c=(x^c \setminus y^c)\cup (y^c\setminus x^c)$, the Jaccard index as $\text{IoU}^c = \frac{|v^c|}{|u^c|}$, and the Dice score as $\text{Dice}^c=\frac{2\text{IoU}^c}{1+\text{IoU}^c}=\frac{2|v^c|}{|x^c|+|y^c|}$. In what follows, we will represent sets as binary vectors $x^c,y^c,u^c,v^c,m^c \in \{0,1\}^p$ and denote $|x^c|=\sum_{i=1}^p x^c_i$ the cardinality of the relevant set. Moreover, when the context is clear, we will drop the superscript $c$.

\subsection{Existing Extensions}
If we want to optimize the Dice score, hence, minimize the Dice loss $\Delta_{\text{Dice}} = 1-\text{Dice}$ in a continuous setting, we need to extend $\Delta_{\text{Dice}}$ with $\overline{\Delta}_{\text{Dice}}$ such that it can take any predicted segmentation $\tilde{x}\in[0,1]^p$ as input. Hereinafter, when there is no ambiguity, we will use $x$ and $\tilde{x}$ interchangeably. 

The soft Dice loss (SDL) \cite{SoftDiceLossSudreMICCAIWorkshop2017} extends $\Delta_{\text{Dice}}$ by realizing that when $x,y \in \{0,1\}^p$, $|v|=\langle x, y\rangle$, $|x|=\|x\|_1$ and $|y|=\|y\|_1$. Therefore, SDL replaces the set notation with vector functions:
\begin{equation}
   \overline{\Delta}_{\text{SDL,$L^1$}}: x\in[0,1]^p, y\in\{0,1\}^p \mapsto 1 - \frac{2\langle x, y\rangle}{\|x\|_1+\|y\|_1}.
\end{equation}
The soft Jaccard loss (SJL) \cite{OptimalNowozinCVPR2014,SoftJaccardRahmanISVC2016} can be defined in a similar way:
\begin{equation}
   \overline{\Delta}_{\text{SJL,$L^1$}}: x\in[0,1]^p, y\in\{0,1\}^p \mapsto 1 - \frac{\langle x, y\rangle}{\|x\|_1+\|y\|_1-\langle x, y\rangle}.
\end{equation}
The $L^1$ norm can be replaced with the squared $L^2$ norm \cite{OptimizationTMI2020,V-NetMilletari3DV2016}:
\begin{align}
   \overline{\Delta}_{\text{SDL},L^2}&: x\in[0,1]^p, y\in\{0,1\}^p \mapsto  1 - \frac{2\langle x, y\rangle}{\|x\|_2+\|y\|_2} \\
   \overline{\Delta}_{\text{SJL},L^2}&: x\in[0,1]^p, y\in\{0,1\}^p \mapsto  1 - \frac{\langle x, y\rangle}{\|x\|_2^2+\|y\|_2^2-\langle x, y\rangle}.
\end{align}

A major limitation of loss functions based on $L^1$ relaxations, including SDL, SJL, the soft Tversky loss \cite{SoftTverskyLossSalehiMICCAIWorkshop2017} and the focal Tversky loss \cite{FocalTverskyLossAbrahamISBI2019}, as well as those relying on the Lovasz extension, such as the Lovasz hinge loss \cite{LovaszHingeYuTPAMI2018}, the Lovasz-Softmax loss \cite{Lovasz-softmaxLossBermanCVPR2018} and the PixIoU loss \cite{PixIoUYuICML2021}, is that they cannot handle soft labels \cite{JMLsWangNeurIPS2023}. That is, when $y$ is also in $[0,1]^p$. In particular, both SDL and SJL do not reach their minimum at $x=y$, but instead they drive $x$ towards the vertices $\{0,1\}^p$ \cite{TheoreticalBertelsMIA2021,JMLsWangNeurIPS2023,NoisyNordstromarXiv2023}. Take for example $y=0.5$; it is straightforward to verify that SDL achieves its minimum at $x=1$, which is clearly erroneous. Loss functions that utilize $L^2$ relaxations \cite{V-NetMilletari3DV2016,OptimizationTMI2020} do not exhibit this problem \cite{JMLsWangNeurIPS2023}, but they are less commonly employed in practice and are shown to be inferior to their $L^1$ counterparts \cite{OptimizationTMI2020,JMLsWangNeurIPS2023}. 

To address this, Wang et al. \cite{JMLsWangNeurIPS2023} proposed two variants of SJL termed as Jaccard Metric Losses (JMLs). These two variants, $\overline{\Delta}_{\text{JML1}}$ and $\overline{\Delta}_{\text{JML2}}: [0,1]^p\times[0,1]^p \rightarrow [0,1]$ are defined as
\begin{align}\label{eq:jml}
    \overline{\Delta}_{\text{JML1}} = 1 - \frac{\|x\|_1+\|y\|_1-\|x-y\|_1}{\|x\|_1+\|y\|_1+\|x-y\|_1}, \quad
    \overline{\Delta}_{\text{JML2}} = 1 - \frac{\langle x, y\rangle}{\langle x, y\rangle+\|x-y\|_1}.
\end{align}

Here, $\|x-y\|_1$ represents the symmetric difference. JMLs are shown to be a metric on $[0,1]^p$, according to the definition below.
\begin{definition}[Metric \cite{EncyclopediaDeza2009}]
A mapping $f: [0,1]^p \times [0,1]^p \rightarrow \mathbb{R}$ is called a metric if it satisfies the following conditions for all $a,b,c \in [0,1]^p$:
\begin{enumerate}[label=(\roman*)]
    \item (Reflexivity). $f(a, a) = 0$.
    \item (Positivity). $a\neq b \implies f(a,b)>0$.
    \item (Symmetry). $f(a,b) = f(b,a)$.
    \item (Triangle inequality). $f(a,c) \leq f(a,b) + f(b,c)$.
\end{enumerate}
\end{definition}
Note that reflexivity and positivity jointly imply $x=y \Leftrightarrow f(x,y) = 0$, hence, a loss function that satisfies these conditions will be compatible with soft labels. 

\subsection{Dice Semimetric Losses}
We focus here on the Dice loss. For the derivation of the Tversky loss and the focal Tversky loss, please refer to Appendix \ref{app:tversky}.

Since $\text{Dice}=\frac{2\text{IoU}}{1+\text{IoU}}\Rightarrow 1-\text{Dice}=\frac{1-\text{IoU}}{2-(1-\text{IoU})}$, we have $\overline{\Delta}_{\text{Dice}}=\frac{\overline{\Delta}_{\text{IoU}}}{2- \overline{\Delta}_{\text{IoU}}}$. There exist several alternatives to define $\overline{\Delta}_{\text{IoU}}$, but not all of them are feasible, e.g., SJL. Generally, it is easy to verify the following proposition.
\begin{proposition}\label{thm:DiceIoUcompatibility}
$\overline{\Delta}_{\text{Dice}}$ satisfies reflexivity and positivity iff $\overline{\Delta}_{\text{IoU}}$ does.
\end{proposition}

Among the definitions of $\overline{\Delta}_{\text{IoU}}$, Wang et al. \cite{JMLsWangNeurIPS2023} found only two candidates as defined in Eq. (\ref{eq:jml}) satisfy reflexivity and positivity. Following Proposition~\ref{thm:DiceIoUcompatibility}, we transform these two IoU losses and define Dice semimetric losses (DMLs) $\overline{\Delta}_{\text{DML1}}, \overline{\Delta}_{\text{DML2}}: [0,1]^p\times[0,1]^p \rightarrow [0,1]$ as  
\begin{align}
    \overline{\Delta}_{\text{DML1}} = 1-\frac{\|x\|_1+\|y\|_1-\|x-y\|_1}{\|x\|_1+\|y\|_1}, \quad
    \overline{\Delta}_{\text{DML2}} = 1-\frac{2\langle x, y\rangle}{2\langle x, y\rangle+\|x-y\|_1}.
\end{align}

$\Delta_{\text{Dice}}$ that is defined over integers does not satisfy the triangle inequality \cite{RelaxedGrageraTCS2018}, which is shown to be helpful in KD \cite{JMLsWangNeurIPS2023}. Nonetheless, we can consider a weaker form of the triangle inequality:
\begin{equation}
    f(a,c) \leq \rho (f(a,b) + f(b,c)).
\end{equation}
Functions that satisfy the relaxed triangle inequality for some fixed scalar $\rho$ and conditions (i)-(iii) of a metric are called semimetrics. $\Delta_{\text{Dice}}$ is a semimetric on $\{0,1\}^p$ \cite{RelaxedGrageraTCS2018}. $\overline{\Delta}_{\text{DML1}}$ and $\overline{\Delta}_{\text{DML2}}$, which extend $\Delta_{\text{Dice}}$ to $[0,1]^p$, remain semimetrics in the continuous space as the following theorem shows.
\begin{theorem} \label{thm:rhodice}
$\overline{\Delta}_{\text{DML1}}$ and $\overline{\Delta}_{\text{DML2}}$ are semimetrics on $[0,1]^p$.
\end{theorem}

The proof can be found in Appendix \ref{app:thm_semi}. Moreover, DMLs have properties that are similar to JMLs and they are presented as follows.
\begin{theorem} \label{thm:eqneq}
$\forall x\in[0,1]^p, y\in \{0,1\}^p$ and $x\in \{0,1\}^p, y\in [0,1]^p$, $\overline{\Delta}_{\text{DML1}}=\overline{\Delta}_{\text{DML2}}=\overline{\Delta}_{\text{SDL,$L^1$}}$. $\forall x, y\in[0,1]^p, \overline{\Delta}_{\text{DML1,$L^2$}}=\overline{\Delta}_{\text{DML2,$L^2$}}=\overline{\Delta}_{\text{SDL,$L^2$}}$. $\exists x, y \in [0,1]^p, \overline{\Delta}_{\text{DML1}}\neq\overline{\Delta}_{\text{DML2}}\neq\overline{\Delta}_{\text{SDL,$L^1$}}$.
\end{theorem}
\begin{theorem} \label{thm:smaller}
$\forall x, y\in[0,1]^p$, $\overline{\Delta}_{\text{DML1}}\leq \overline{\Delta}_{\text{DML2}}$.
\end{theorem}

The proofs are analogous to those given in \cite{JMLsWangNeurIPS2023}. In particular, Theorem~\ref{thm:eqneq} indicates that when only hard labels are presented, $\overline{\Delta}_{\text{DML1}},\overline{\Delta}_{\text{DML2}}$, $\overline{\Delta}_{\text{SDL},L^1}$ are identical. Similarly, we can substitute the $L^1$ norm in DMLs with the squared $L^2$ norm, and $\overline{\Delta}_{\text{DML1,$L^2$}}, \overline{\Delta}_{\text{DML2,$L^2$}}$, $\overline{\Delta}_{\text{SDL},L^2}$ become the same. Therefore, we can safely replace the existing implementation of SDL with DMLs and no change will be incurred.

\section{Experiments}\label{sec:experiments}
In this section, we provide empirical evidence of the benefits of using soft labels. In particular, using QUBIQ \cite{QUBIQMenze2020}, which contains multi-rater information, we show that models trained with averaged annotation maps can significantly surpass those trained with majority votes and random selections. Leveraging LiTS \cite{LiTSBilicMIA2023} and KiTS \cite{KiTSHellerMIA2021}, we illustrate the synergistic effects of integrating LS and KD with DMLs. 

\subsection{Datasets}
\noindent \textbf{QUBIQ} is a recent challenge held at MICCAI 2020 and 2021, specifically designed to evaluate the inter-rater variability in medical imaging. Following \cite{MRNetJiCVPR2021,UsingSilvaMICCAIWorkshop2021}, we use QUBIQ 2020, which contains 7 segmentation tasks in 4 different CT and MR datasets: Prostate (55 cases, 2 tasks, 6 raters), Brain Growth (39 cases, 1 task, 7 raters), Brain Tumor (32 cases, 3 tasks, 3 raters), and Kidney (24 cases, 1 task, 3 raters). For each dataset, we calculate the average Dice score between each rater and the majority votes in Table \ref{tb:ratio}. In some datasets, such as Brain Tumor T2, the inter-rater disagreement can be quite substantial. In line with \cite{MRNetJiCVPR2021}, we resize all images to $256\times 256$. 

\noindent \textbf{LiTS} contains 201 high-quality CT scans of liver tumors. Out of these, 131 cases are designated for training and 70 for testing. As the ground-truth labels for the test set are not publicly accessible, we only use the training set. Following \cite{EMKDQinTMI2021}, all images are resized to $512 \times 512$ and the HU values of CT images are windowed to the range of [-60, 140].

\noindent \textbf{KiTS} includes 210 annotated CT scans of kidney tumors from different patients. In accordance with \cite{EMKDQinTMI2021}, all images are resized to $512 \times 512$ and the HU values of CT images are windowed to the range of [-200, 300].

\begin{table}
\tiny
\centering
\caption{The number of raters and the averaged Dice score between each rater and the majority votes for each QUBIQ dataset. D1: Prostate T1, D2: Prostate T2, D3: Brain Growth T1, D4: Brain Tumor T1, D5: Brain Tumor T2, D6: Brain Tumor T3, D7: Kidney T1.} \label{tb:ratio}
\begin{tabular}{cccccccc} \hlineB{2}
Dataset & D1 & D2 & D3 & D4 & D5 & D6 & D7 \\ 
\hline
\# Raters & 6 & 6 & 7 & 3 & 3 & 3 & 3 \\
Dice (\%) & 96.49 & 92.17 & 91.20 & 95.44 & 68.73 & 92.71 & 97.41 \\
\hlineB{2}
\end{tabular}
\end{table}

\subsection{Implementation Details}
We adopt a variety of backbones including ResNet50/18 \cite{ResNetHeCVPR2016}, EfficientNetB0 \cite{EfficientNetTanICML2019} and MobileNetV2 \cite{MobileNetV2SandlerCVPR2018}. All these models that have been pretrained on ImageNet \cite{ImageNetDengCVPR2009} are provided by timm library \cite{timmWightman2019}. We consider both UNet \cite{U-NetRonnebergerMICCAI2015} and DeepLabV3+ \cite{DeepLabV3+ChenECCV2018} as the segmentation method.

We train the models using SGD with an initial learning rate of 0.01, momentum of 0.9, and weight decay of 0.0005. The learning rate is decayed in a poly policy with an exponent of 0.9. The batch size is set to 8 and the number of epochs is 150 for QUBIQ, 60 for both LiTS and KiTS. We leverage a mixture of CE and DMLs weighted by 0.25 and 0.75, respectively. Unless otherwise specified, we use $\overline{\Delta}_{\text{DML1}}$ by default.

In this work, we are mainly interested in how models can benefit from the use of soft labels. The superiority of SDL over CE has been well established in the medical imaging community \cite{OptimizationTMI2020,nnU-NetIsenseeNatureMethods2021}, and our preliminary experiments also confirm this, as shown in Table \ref{tb:ceDMLs} (Appendix \ref{app:tables}). Therefore, we do not include any further comparison with CE in this paper.

\subsection{Evaluation Metrics} \label{sec:bdice}
We report both the Dice score and the expected calibration error (ECE) \cite{CalibrationGuoICML2017}. For QUBIQ experiments, we additionally present the binarized Dice score (BDice), which is the official evaluation metrics used in the QUBIQ challenge. To compute BDice, both predictions and soft labels are thresholded at different probability levels (0.1, 0.2, ..., 0.8, 0.9). We then compute the Dice score at each level and average these scores with all thresholds. 

For all experiments, we conduct 5-fold cross validation, making sure that each case is presented in exactly one validation set, and report the mean values in the aggregated validation set. We perform statistical tests according to the procedure detailed in \cite{OptimizationTMI2020} and highlight results that are significantly superior (with a significance level of 0.05) in red. 

\subsection{Results on QUBIQ}
In Table \ref{tb:qubiq}, we compare different training methods on QUBIQ using UNet-ResNet50. This comparison includes both hard labels, obtained through (i) majority votes \cite{LabelLemayMELBA2023} and (ii) random sampling each rater's annotation \cite{ImprovingJensenMICCAI2019}, as well as soft labels derived from (i) averaging across all annotations \cite{SoftSegGrosMIA2021,UsingSilvaMICCAIWorkshop2021,LabelLemayMELBA2023} and (ii) label smoothing \cite{InceptionV2V3SzegedyCVPR2016}. 

In the literature \cite{SoftSegGrosMIA2021,UsingSilvaMICCAIWorkshop2021,LabelLemayMELBA2023}, annotations are usually averaged with uniform weights. We additionally consider weighting each rater's annotation by its Dice score with respect to the majority votes, so that a rater who deviates far from the majority votes receives a low weight. Note that for all methods, the Dice score and ECE are computed with respect to the majority votes, while BDice is calculated as illustrated in Section \ref{sec:bdice}.

Generally, models trained with soft labels exhibit improved accuracy and calibration. In particular, averaging annotations with uniform weights obtains the highest BDice, while a weighted average achieves the highest Dice score. It is worth noting that the weighted average significantly outperforms the majority votes in terms of the Dice score which is evaluated based on the majority votes themselves. We hypothesize that this is because soft labels contain extra inter-rater information, which can ease the network optimization at those ambiguous regions. Overall, we find the weighted average outperforms other methods, with the exception of Brain Tumor T2, where there is a high degree of disagreement among raters.

We compare our method with state-of-the-art (SOTA) methods using UNet-ResNet50 in Table \ref{tb:qubiq_sota}. In our method, we average annotations with uniform weights for Brain Tumor T2 and with each rater's Dice score for all other datasets. Our method, which simply averages annotations to produce soft labels obtains superior results compared to methods that adopt complex architectures or training techniques.

\begin{table}[h]
\tiny
\centering
\caption{Comparing hard labels with soft labels on QUBIQ using UNet-ResNet50.} \label{tb:qubiq}
\begin{tabular}{cccccccccc}
\hlineB{2}
Dataset & Metric & Majority & Random & Uniform & Weighted & LS \\ 
\hline
\multicolumn{1}{c}{\multirow{3}{*}{Prostate T1}}
& Dice (\%) & \cellcolor{red!15}{95.65} & \cellcolor{red!15}{95.80} & \cellcolor{red!15}{95.74} & \cellcolor{red!15}{95.99} & \cellcolor{red!15}{95.71} \\
& BDice (\%) & 94.72 & \cellcolor{red!15}{95.15} & \cellcolor{red!15}{95.19} & \cellcolor{red!15}{95.37} & \cellcolor{red!15}{94.91} \\
& ECE (\%) & 0.51 & 0.39 & \cellcolor{red!15}{0.22} & \cellcolor{red!15}{0.20} & 0.36 \\
\hline
\multicolumn{1}{c}{\multirow{3}{*}{Prostate T2}}
& Dice (\%) & \cellcolor{red!15}{89.39} & 88.87 & \cellcolor{red!15}{89.57} & \cellcolor{red!15}{89.79} & \cellcolor{red!15}{89.82} \\
& BDice (\%) & 88.31 & 88.23 & \cellcolor{red!15}{89.35} & \cellcolor{red!15}{89.66} & 88.85 \\
& ECE (\%) & 0.52 & 0.47 & \cellcolor{red!15}{0.26} & \cellcolor{red!15}{0.25} & 0.41 \\
\hline
\multicolumn{1}{c}{\multirow{3}{*}{Brain Growth}}
& Dice (\%) & \cellcolor{red!15}{91.09} & 90.65 & \cellcolor{red!15}{90.94} & \cellcolor{red!15}{91.46} & \cellcolor{red!15}{91.23} \\
& BDice (\%) & 88.72 & 88.81 & \cellcolor{red!15}{89.89} & \cellcolor{red!15}{90.40} & \cellcolor{red!15}{89.88} \\
& ECE (\%) & 1.07 & 0.85 & \cellcolor{red!15}{0.27} & 0.34 & 0.41 \\
\hline
\multicolumn{1}{c}{\multirow{3}{*}{Brain Tumor T1}}
& Dice (\%) & 86.46 & \cellcolor{red!15}{87.24} & \cellcolor{red!15}{87.74} & \cellcolor{red!15}{87.78} & \cellcolor{red!15}{87.84} \\
& BDice (\%) & 85.74 & \cellcolor{red!15}{86.59} & \cellcolor{red!15}{86.67} & \cellcolor{red!15}{86.92} & \cellcolor{red!15}{86.91} \\
& ECE (\%) & 0.62 & 0.55 & \cellcolor{red!15}{0.38} & \cellcolor{red!15}{0.36} & \cellcolor{red!15}{0.37} \\
\hline
\multicolumn{1}{c}{\multirow{3}{*}{Brain Tumor T2}}
& Dice (\%) & 58.58 & 48.86 & 52.42 & \cellcolor{red!15}{61.01} & \cellcolor{red!15}{61.23} \\
& BDice (\%) & 38.68 & 49.19 & \cellcolor{red!15}{55.11} & 44.23 & 40.61 \\
& ECE (\%) & 0.25 & 0.81 & 0.74 & 0.26 & \cellcolor{red!15}{0.22} \\
\hline
\multicolumn{1}{c}{\multirow{3}{*}{Brain Tumor T3}}
& Dice (\%) & 53.54 & 54.64 & 53.45 & \cellcolor{red!15}{56.75} & \cellcolor{red!15}{57.01} \\
& BDice (\%) & 52.33 & 53.53 & 51.98 & 53.90 & \cellcolor{red!15}{55.26} \\
& ECE (\%) & 0.17 & 0.17 & 0.14 & \cellcolor{red!15}{0.09} & \cellcolor{red!15}{0.11} \\
\hline
\multicolumn{1}{c}{\multirow{3}{*}{Kidney}}
& Dice (\%) & 62.96 & 68.10 & 71.33 & \cellcolor{red!15}{76.18} & 71.21 \\
& BDice (\%) & 62.47 & 67.69 & 70.82 & \cellcolor{red!15}{75.67} & 70.41 \\
& ECE (\%) & 0.88 & 0.78 & 0.67 & \cellcolor{red!15}{0.53} & 0.62 \\
\hline
\multicolumn{1}{c}{\multirow{3}{*}{All}}
& Dice (\%) & 76.80 & 76.30 & 77.31 & \cellcolor{red!15}{79.85} & 79.15 \\
& BDice (\%) & 72.99 & 75.59 & \cellcolor{red!15}{77.00} & 76.59 & 75.26 \\
& ECE (\%) & 0.57 & 0.57 & 0.38 & \cellcolor{red!15}{0.29} & 0.35 \\
\hlineB{2}
\end{tabular}
\end{table}

\begin{table}[!h]
\tiny
\centering
\caption{Comparing SOTA methods with ours on QUBIQ using UNet-ResNet50. All results are BDice (\%).} \label{tb:qubiq_sota}
\begin{tabular}{ccccccc}
\hlineB{2}
Dataset & Dropout \cite{DropoutGalICML2016} & Multi-head \cite{BIWDNGuanAAAI2018} & MRNet \cite{MRNetJiCVPR2021} & SoftSeg \cite{SoftSegGrosMIA2021,LabelLemayMELBA2023} & Ours \\ 
\hline
Prostate T1 & \cellcolor{red!15}{94.91} & \cellcolor{red!15}{95.18} & \cellcolor{red!15}{95.21} & \cellcolor{red!15}{95.02} & \cellcolor{red!15}{95.37} \\
Prostate T2 & 88.43 & 88.32 & 88.65 & 88.81 & \cellcolor{red!15}{89.66} \\
Brain Growth & 88.86 & 89.01 & 89.24 & 89.36 & \cellcolor{red!15}{90.40} \\
Brain Tumor T1 & 85.98 & \cellcolor{red!15}{86.45} & \cellcolor{red!15}{86.33} & \cellcolor{red!15}{86.41} & \cellcolor{red!15}{86.92} \\
Brain Tumor T2 & 48.04 & 51.17 & 51.82 & 52.56 & \cellcolor{red!15}{55.11} \\
Brain Tumor T3 & 52.49 & \cellcolor{red!15}{53.68} & \cellcolor{red!15}{54.22} & 52.43 & \cellcolor{red!15}{53.90} \\
Kidney & 66.53 & 68.00 & 68.56 & 69.83 & \cellcolor{red!15}{75.67} \\
All & 75.03 & 75.97 & 76.18 & 76.34 & \cellcolor{red!15}{78.14} \\
\hlineB{2}
\end{tabular}
\end{table}

\subsection{Results on LiTS and KiTS}
Wang et al. \cite{JMLsWangNeurIPS2023} empirically found that a well-calibrated teacher can distill a more accurate student. Concurrently, Menon et al. \cite{AStatisticalMenonICML2021} argued that the effectiveness of KD arises from the teacher providing an estimation of the Bayes class-probabilities $p^*(y|x)$ and this can lower the variance of the student's empirical loss. 

In line with these findings, in Appendix \ref{app:thm:ce}, we prove $|\mathbb{E}[p^*(y|x) - f(x)]| \leq \mathbb{E}[|\mathbb{E}[y|f(x)]-f(x)|]$. That is, the bias of the estimation is bounded above by the calibration error and this explains why the calibration of the teacher would be important for the student. Inspired by this, we apply a recent kernel density estimator (KDE) \cite{KDE-XEPopordanoskaNeurIPS2022} that provides consistent estimation of $\mathbb{E}[y|f(x)]$. We then adopt it as a post-hoc calibration method to replace the temperature scaling to calibrate the teacher in order to improve the performance of the student. For more details of KDE, please refer to Appendix \ref{app:kde}.

In Table \ref{tb:lits_kits}, we compare models trained with hard labels, LS \cite{InceptionV2V3SzegedyCVPR2016} and KD \cite{KDHintonNeurIPSWorkshop2015} on LiTS and KiTS, respectively. For all KD experiments, we use UNet-ResNet50 as the teacher. Again, we obtain noticeable improvements in both the Dice score and ECE. It is worth noting that for UNet-ResNet18 and UNet-EfficientNetB0 on LiTS, the student's Dice score exceeds that of the teacher.

\begin{table}[h]
\tiny
\centering
\caption{Comparing hard labels with LS and KD on LiTS and KiTS.} \label{tb:lits_kits}
\begin{tabular}{ccccccccc}
\hlineB{2}
\multirow{2}{*}{Method} & \multirow{2}{*}{Backbone} & \multirow{2}{*}{Metric} & \multicolumn{3}{c}{LiTS} & \multicolumn{3}{c}{KiTS} \\ \cline{4-9} & & & Hard & LS & KD & Hard & LS & KD \\ \hline
\multirow{2}{*}{UNet} & \multirow{2}{*}{ResNet50} 
& Dice (\%) & 59.79 & \cellcolor{red!15}{60.59} & - & 72.66 & \cellcolor{red!15}{73.92} & - \\
& & ECE (\%) & 0.51 & \cellcolor{red!15}{0.49} & - & 0.39 & \cellcolor{red!15}{0.33} & - \\ \hline  
\multirow{2}{*}{UNet} & \multirow{2}{*}{ResNet18} 
& Dice (\%) & 57.92 & 58.60 & \cellcolor{red!15}{60.30} & 67.96 & 69.09 & \cellcolor{red!15}{71.34} \\
& & ECE (\%) & 0.52 & \cellcolor{red!15}{0.48} & 0.50 & 0.44 & \cellcolor{red!15}{0.38} & 0.44 \\ \hline   
\multirow{2}{*}{UNet} & \multirow{2}{*}{EfficientNetB0}  
& Dice (\%) & 56.90 & 57.66 & \cellcolor{red!15}{60.11} & 70.31 & 71.12 & \cellcolor{red!15}{71.73} \\
& & ECE (\%) & 0.56 & \cellcolor{red!15}{0.47} & 0.52 & 0.39 & \cellcolor{red!15}{0.35} & 0.39 \\ \hline   
\multirow{2}{*}{UNet} & \multirow{2}{*}{MobileNetV2}  
& Dice (\%) & 56.16 & 57.20 & \cellcolor{red!15}{58.92} & 67.46 & 68.19 & \cellcolor{red!15}{68.85} \\
& & ECE (\%) & 0.54 & \cellcolor{red!15}{0.48} & 0.50 & 0.42 & \cellcolor{red!15}{0.38} & 0.41 \\ \hline   
\multirow{2}{*}{DeepLabV3+} & \multirow{2}{*}{ResNet18}  
& Dice (\%) & 56.10 & 57.07 & \cellcolor{red!15}{59.12} & 69.95 & \cellcolor{red!15}{70.61} & \cellcolor{red!15}{70.80} \\
& & ECE (\%) & 0.53 & \cellcolor{red!15}{0.50} & 0.52 & 0.40 & \cellcolor{red!15}{0.38} & 0.40 \\ 
\hlineB{2}
\end{tabular}
\end{table}

\subsection{Ablation Studies}
In Table \ref{tb:DMLs} (Appendix \ref{app:tables}), we compare SDL with DMLs. For QUBIQ, we train UNet-ResNet50 with soft labels obtained from weighted average and report BDice. For LiTS and KiTS, we train UNet-ResNet18 with KD and present the Dice score. For a fair comparison, we disable KDE in all KD experiments. 

We find models trained with SDL can still benefit from soft labels to a certain extent because (i) models are trained with a mixture of CE and SDL, and CE is compatible with soft labels; (ii) although SDL pushes predictions towards vertices, it can still add some regularization effects in a binary segmentation setting. However, SDL is notably outperformed by DMLs. As for DMLs, we find $\overline{\Delta}_{\text{DML1}}$ is slightly superior to $\overline{\Delta}_{\text{DML2}}$ and recommend using $\overline{\Delta}_{\text{DML1}}$ in practice.

In Table \ref{tb:kdloss} (Appendix \ref{app:tables}), we ablate the contribution of each KD term on LiTS and KiTS with a UNet-ResNet18 student. In the table, CE and DML represent adding the CE and DML term between the teacher and the student, respectively. In Table \ref{tb:bandwidth} (Appendix \ref{app:tables}), we illustrate the effect of bandwidth that controls the smoothness of KDE. Results shown in the tables verify the effectiveness of the proposed loss and the KDE method.

\section{Future Works}
In this study, our focus is on extending the Dice loss within the realm of medical image segmentation. It may be intriguing to apply DMLs in the context of long-tailed classification \cite{DiceLossLiACL2020}. Additionally, while we employ DMLs in the label space, it holds potential for measuring the similarity of two feature vectors \cite{MasKDHuangICLR2023}, for instance, as an alternative to the $L^p$ norm or cosine similarity.

\section{Conclusion}\label{sec:conclusion}
In this work, we introduce Dice semimetrics losses (DMLs), which are identical to the soft Dice loss (SDL) in a standard setting with hard labels, but are fully compatible with soft labels. Our extensive experiments on the public QUBIQ, LiTS and KiTS benchmarks validate that incorporating soft labels leads to higher Dice score and lower calibration error, indicating that these losses can find wide application in diverse medical image segmentation problems. Hence, we suggest to replace the existing implementation of SDL with DMLs.

\section*{Acknowledgements}
We acknowledge support from the Research Foundation - Flanders (FWO) through project numbers G0A1319N and S001421N, and funding from the Flemish Government under the Onderzoeksprogramma Artifici\"{e}le Intelligentie (AI) Vlaanderen programme. The resources and services used in this work were provided by the VSC (Flemish Supercomputer Center), funded by the Research Foundation - Flanders (FWO) and the Flemish Government.

\bibliographystyle{splncs04}
\bibliography{miccai2023}

\clearpage

\appendix
\section{Theorem \ref{thm:rhodice}} \label{app:thm_semi}
\begin{proof}
We omit the subscript since the proof for $\overline{\Delta}_{\text{DML1}}$ and $\overline{\Delta}_{\text{DML2}}$ are identical. Note that $0\leq \overline{\Delta}_{\text{DML}}=\frac{\overline{\Delta}_{\text{JML}}}{2- \overline{\Delta}_{\text{JML}}}\leq \overline{\Delta}_{\text{JML}}\leq 1$ and $\overline{\Delta}_{\text{JML}}$ is a metric, thus satisfying the triangle inequality.
    \begin{align}
        & \overline{\Delta}_{\text{DML}}(a,c) \leq \rho(\overline{\Delta}_{\text{DML}}(a,b)+\overline{\Delta}_{\text{DML}}(b,c)) \\
        \Rightarrow & \frac{\overline{\Delta}_{\text{JML}}(a,c)}{2- \overline{\Delta}_{\text{JML}}(a,c)} \leq \frac{\rho\overline{\Delta}_{\text{JML}}(a,b)}{2- \overline{\Delta}_{\text{JML}}(a,b)} + \frac{\rho\overline{\Delta}_{\text{JML}}(b,c)}{2- \overline{\Delta}_{\text{JML}}(b,c)} \\
        \Rightarrow & \frac{\overline{\Delta}_{\text{JML}}(a,b) + \overline{\Delta}_{\text{JML}}(b,c)}{2 - \overline{\Delta}_{\text{JML}}(a,b) - \overline{\Delta}_{\text{JML}}(b,c)} \leq \frac{\rho\overline{\Delta}_{\text{JML}}(a,b)}{2- \overline{\Delta}_{\text{JML}}(a,b)} + \frac{\rho\overline{\Delta}_{\text{JML}}(b,c)}{2- \overline{\Delta}_{\text{JML}}(b,c)} \\
        \Rightarrow & \overline{\Delta}_{\text{JML}}(a,b) \Big(\frac{1}{2 - \overline{\Delta}_{\text{JML}}(a,b) - \overline{\Delta}_{\text{JML}}(b,c)} - \frac{\rho}{2-\overline{\Delta}_{\text{JML}}(a,b)}\Big) \leq \nonumber \\
        & \overline{\Delta}_{\text{JML}}(b,c) \Big(\frac{\rho}{2-\overline{\Delta}_{\text{JML}}(b,c)} - \frac{1}{2-\overline{\Delta}_{\text{JML}}(a,b)-\overline{\Delta}_{\text{JML}}(b,c)} \Big) \label{eq:DMLleftright}.
    \end{align}
    
    Assume that $\overline{\Delta}_{\text{DML}}(a,b) + \overline{\Delta}_{\text{DML}}(b,c) \leq \frac{1}{\rho}$, otherwise $\rho$-relaxed triangle inequality trivially holds. Now we show that the left-hand side of Eq. (\ref{eq:DMLleftright}) is no greater than 0.
    \begin{align}
    & \frac{\rho}{2-\overline{\Delta}_{\text{JML}}(a,b)} \geq \frac{1}{2 - \overline{\Delta}_{\text{JML}}(a,b) - \overline{\Delta}_{\text{JML}}(b,c)} \\
    \Rightarrow & \frac{\rho+\overline{\Delta}_{\text{DML}}(a,b)}{1+\overline{\Delta}_{\text{DML}}(a,b)} \geq \frac{1+(1+\rho)\overline{\Delta}_{\text{DML}}(b,c)}{1+\overline{\Delta}_{\text{DML}}(b,c)} \label{eq:DMLright1} \\
    \Rightarrow & \frac{\rho+\overline{\Delta}_{\text{DML}}(a,b)}{1+\overline{\Delta}_{\text{DML}}(a,b)} \geq \frac{1+(1+\rho)(\frac{1}{\rho}-\overline{\Delta}_{\text{DML}}(a,b))}{1+(\frac{1}{\rho}-\overline{\Delta}_{\text{DML}}(a,b))} \label{eq:DMLright2} \\
    \Rightarrow & \rho^2 -\rho - 1 + \rho^2\overline{\Delta}_{\text{DML}}^2(a,b) \geq 0 \\
    \Rightarrow & \rho \geq \frac{1+\sqrt{5}}{2} \approx 1.62.
    \end{align}
    where in Eq. (\ref{eq:DMLright2}), we use the fact the right hand side of Eq. (\ref{eq:DMLright1}) is an increasing function of $\overline{\Delta}_{\text{DML}}(b,c)$. Similarly, we can show that the right-hand side of Eq. (\ref{eq:DMLleftright}) is no less than 0 when $\rho \geq \frac{1+\sqrt{5}}{2}$. Note that for $a=[0,1], b=[1,1], c=[1,0]$, $\overline{\Delta}_{\text{DML}}(a,c)=1, \overline{\Delta}_{\text{DML}}(a,b)=\overline{\Delta}_{\text{DML}}(b,c)=\frac{1}{3}$, so $\rho$ should at least be $\frac{3}{2}$. We leave the tight bound of $\rho$ as an open problem. $\hfill\square$
    \end{proof}

\section{Theorem \ref{thm:ce}} \label{app:thm:ce}
\begin{theorem} \label{thm:ce}
 $|\mathbb{E}[p^*(y|x) - f(x)]| \leq \mathbb{E}[|\mathbb{E}[y|f(x)]-f(x)|]$.
\end{theorem}
\begin{proof}
Subscripts are omitted for simplicity and all expectations are over $x, y \sim \mathbb{P}_{x,y}$, the (unknown) data generating distribution. Note that $\mathbb{E}[|\mathbb{E}[y|f(x)]-f(x)|]$ is the definition of the calibration error. The proof is similar to \cite{OnPopordanoskaMICCAI2021}:
\begin{align}
    & |\mathbb{E}[p^*(y|x) - f(x)]| \\
    = & |\mathbb{E}[p^*(y|x) - \mathbb{E}[y|f(x)] + \mathbb{E}[y|f(x)] - f(x)]| \\
    = & |\mathbb{E}[p^*(y|x)] - \mathbb{E}[\mathbb{E}[y|f(x)]] + \mathbb{E}[\mathbb{E}[y|f(x)] - f(x)]| \\
    = & |\underbrace{\mathbb{E}[p^*(y|x) - y]}_{=0} + \mathbb{E}[\mathbb{E}[y|f(x)] - f(x)]| \\
    \leq & \mathbb{E}[|\mathbb{E}[y|f(x)] - f(x)|]. \label{eq:jensen}
\end{align}
Eq. (\ref{eq:jensen}) follows from Jensen's inequality due to the convexity of the absolute value. $\hfill\square$
\end{proof}

\section{Tables} \label{app:tables}
\begin{table}[] 
\tiny
\centering
\caption{Comparing CE with DML on QUBIQ, LiTS and KiTS using UNet-ResNet50. All results are the Dice score (\%).} \label{tb:ceDMLs}
\begin{tabular}{ccc} \hlineB{2}
Dataset & CE & DML \\ \hline
QUBIQ & 74.97 & \cellcolor{red!15}{76.80} \\ 
LiTS & 57.76 & \cellcolor{red!15}{59.79} \\ 
KiTS & 68.52 & \cellcolor{red!15}{72.66} \\ 
\hlineB{2}
\end{tabular}

\vspace{5mm}

\caption{Comparing SDL with DMLs on QUBIQ, LiTS and KiTS. All results are the Dice score (\%).} \label{tb:DMLs}
\begin{tabular}{ccccc}
\hlineB{2}
Dataset & Hard & $\overline{\Delta}_{\text{SDL}}$ & $\overline{\Delta}_{\text{DML1}}$ & $\overline{\Delta}_{\text{DML2}}$ \\ 
\hline
QUBIQ & 72.99 & 73.79 & \cellcolor{red!15}{76.59} & \cellcolor{red!15}{76.42} \\
LiTS & 57.92 & 58.12 & \cellcolor{red!15}{59.31} & \cellcolor{red!15}{59.12} \\
KiTS & 67.96 & 68.26 & \cellcolor{red!15}{69.29} & \cellcolor{red!15}{69.07} \\
\hlineB{2}
\end{tabular}

\vspace{5mm}

\caption{Evaluating each KD term on LiTS and KiTS using a UNet-ResNet18 student. All results are the Dice score (\%).} \label{tb:kdloss}
\begin{tabular}{ccccc}
\hlineB{2}
Dataset & Hard & CE & DML & KDE \\ 
\hline
LiTS & 57.92 & 58.23 & 59.31 & \cellcolor{red!15}{60.30} \\
KiTS & 67.96 & 68.14 & 69.29 & \cellcolor{red!15}{71.34} \\
\hlineB{2}
\end{tabular}

\vspace{5mm}

\caption{Comparing different bandwidths on LiTS and KiTS using a UNet-ResNet18 student. All results are the Dice score (\%).} \label{tb:bandwidth}
\begin{tabular}{cccccccc} \hlineB{2}
Dataset & 0 & 5e-5 & 1e-4 & 5e-4 & 1e-3 & 5e-3 & 1e-2 \\ \hline
LiTS & 59.31 & 59.05 & 59.07 & \cellcolor{red!15}{59.97} & \cellcolor{red!15}{60.30} & 59.56 & 59.62 \\
KiTS & 69.29 & 69.05 & 69.80 & 70.41 & \cellcolor{red!15}{71.34} & 68.75 & 69.18 \\
\hlineB{2}
\end{tabular}
\end{table}

\section{The Compatible Tversky Loss and the Compatible Focal Tversky Loss} \label{app:tversky}
The Tversky index is defined as 
\begin{equation}
    \frac{|v|}{|v| + \alpha (|x| - |v|) + \beta (|y| - |v|)}
\end{equation}
where $\alpha, \beta \geq 0$ controls the magnitude of penalties for false positives and false negatives, respectively. With $\alpha = \beta = 0.5$, the Tversky index becomes the Dice score; with $\alpha = \beta = 1$, it simplifies to the Jaccard index.

Adopting the same idea as SDL and SJL, the soft Tversky loss (STL) \cite{SoftTverskyLossSalehiMICCAIWorkshop2017} is written as 
\begin{equation}
     \overline{\Delta}_{\text{STL}}: x\in[0,1]^p, y\in\{0,1\}^p \mapsto 1 - \frac{\langle x, y\rangle}{\alpha\|x\|_1+\beta\|y\|_1+(1-\alpha-\beta)\langle x, y\rangle}.
\end{equation}

Given that SDL, as a particular instantiation of STL, is not compatible with soft labels, STL inherits this incompatibility as well. To render STL compatible with soft labels, the key lies in expressing $|v|$ as $\frac{1}{2}(\|x\|_1+\|y\|_1-\|x-y\|_1)$. Using this expression, we introduce the compatible Tversky loss (CTL) as
\begin{align}
     \overline{\Delta}_{\text{CTL}}&: x\in[0,1]^p, y\in [0,1]^p \\
     & \mapsto 1 - \frac{\|x\|_1+\|y\|_1-\|x-y\|_1}{2\alpha\|x\|_1+2\beta\|y\|_1+(1-\alpha-\beta)(\|x\|_1+\|y\|_1-\|x-y\|_1)}.
\end{align}

It is worth noting that due to the inherent asymmetry of the Tversky index, $\overline{\Delta}_{\text{CTL}}$ does not satisfy symmetry and therefore cannot be a semimetric. However, it is compatible with soft labels, as shown in the following theorem.
\begin{theorem}
$\overline{\Delta}_{\text{CTL}}$ satisfies reflexivity and positivity.
\end{theorem}
\begin{proof}
Let $S_1=\{i:x_i\geq y_i\}$ and $S_2=\{i:x_i<y_i\}$.
\begin{align}
& \overline{\Delta}_{\text{CTL}}(x,y) = 0 \\
\Rightarrow & (\alpha-\beta)\|x\|_1+(-\alpha+\beta)\|y\|_1+(\alpha+\beta)\|x-y\|_1=0 \\
\Rightarrow & (\alpha-\beta)\sum_{i\in S_1}x_i + (\alpha-\beta)\sum_{i\in S_2}x_i + (-\alpha+\beta)\sum_{i\in S_1}y_i + (-\alpha+\beta)\sum_{i\in S_2}y_i \\
& (\alpha+\beta)\sum_{i\in S_1}x_i - (\alpha+\beta)\sum_{i\in S_2}x_i - (\alpha+\beta)\sum_{i\in S_1}y_i + (\alpha+\beta)\sum_{i\in S_2}y_i = 0 \\
\Rightarrow & \alpha \sum_{i\in S_1}(x_i-y_i) + \beta \sum_{i\in S_2}(y_i-x_i) = 0
\end{align}
where the last equality holds if and only if $x=y$. $\hfill\square$
\end{proof}

Again, following a similar proof as given in \cite{JMLsWangNeurIPS2023}, we can show that CTL is identical to STL in a standard setting with hard labels.
\begin{theorem}
$\forall x\in[0,1]^p,\  y\in \{0,1\}^p$ and $x\in \{0,1\}^p,\  y\in [0,1]^p$, $\overline{\Delta}_{\text{STL}}=\overline{\Delta}_{\text{CTL}}$. $\exists x, y \in [0,1]^p, \overline{\Delta}_{\text{STL}}\neq\overline{\Delta}_{\text{CTL}}$.
\end{theorem}

Building upon the Tversky loss, we incorporate a focal term \cite{FocalLossLinTPAMI2018} as adopted in the focal Tversky loss \cite{FocalTverskyLossAbrahamISBI2019}. We call this new loss as the Compatible Focal Tversky Loss (CFTL) and define it as:
\begin{equation}
    \overline{\Delta}_{\text{CFTL}}: x\in[0,1]^p, y\in [0,1]^p \mapsto \overline{\Delta}_{\text{CTL}}^{\gamma}
\end{equation}
where $\gamma$ is the focal term. With $\gamma > 1$, CFTL focuses more on less accurate predictions that have been misclassified. In Figure \ref{fig:cftl}, we plot the loss value of CFTL with $\alpha=0.7, \beta=0.3$ when the soft label is 0.8. With increased $\gamma$, the loss value flattens when the prediction is near the ground-truth and increases more rapidly when the prediction is far away from the soft label. It is worth noting that this focal term can be seamlessly incorporated into JMLs and DMLs as well.

\begin{figure}
\centering
    \includegraphics[width=0.9\textwidth]{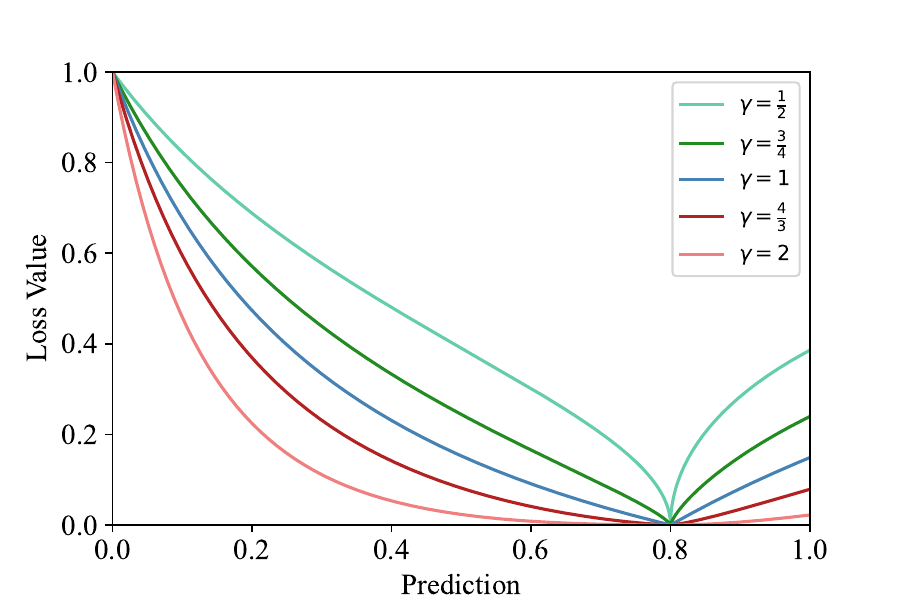}
    \caption{The loss value of CFTL with $\alpha=0.7, \beta=0.3$ when the soft label is 0.8.}
    \label{fig:cftl}
\end{figure}

\section{Calibrated Knowledge Distillation} \label{app:kde}
Wang et al. \cite{JMLsWangNeurIPS2023} empirically demonstrated that a well-calibrated teacher can distill a more accurate student. Concurrently, Menon et al. \cite{AStatisticalMenonICML2021} posited that the effectiveness of KD stems from the teacher's ability to provide an estimation of the Bayes class-probabilities $p^*(y|x)$. According to their argument, such an estimation can reduce the variance in the student's empirical loss. Theorem \ref{thm:ce} implies the bias of the estimation is bounded above by the calibration error. This emphasizes the pivotal role of a teacher's calibration in KD and sheds light on numerous intriguing observations within the field:
\begin{itemize}
    \item Why doesn't a stronger teacher always distill a superior student \cite{DISTHuangNeurIPS2022}? Although a stronger teacher, such as a deeper network, might offer higher accuracy, it could also be less calibrated \cite{CalibrationGuoICML2017}.
    \item Why is it beneficial to smooth a teacher's predictions using temperature scaling \cite{KDHintonNeurIPSWorkshop2015}? Modern neural networks often exhibit overconfidence in their predictions \cite{CalibrationGuoICML2017}. As a remedy, temperature scaling has emerged as a post-hoc calibration method, effectively decreasing a teacher's calibration error \cite{CalibrationGuoICML2017,LTSDingICCV2021,Meta-CalMaICML2021,OnPopordanoskaMICCAI2021,PostRousseauISBI2021}.
    \item Why LS can enhance a model's calibration, but a teacher trained with LS could hurt the student's performance \cite{WhenMullerNeurIPS2019}? We hypothesize that after training a teacher with LS, the optimal temperature in KD should typically be reduced to maintain a low calibration error. Still utilizing a high temperature could result in an excessively smoothed teacher, leading it to become under-confident.
\end{itemize}

Inspired by the fact that we can improve KD by decreasing the teacher's calibration error, we propose the use of $\mathbb{E}[y|f(x)]$ as a distillation signal to supervise the student. Note that by the definition of the calibration error, $\mathbb{E}[y|f(x)]$ represents the optimal recalibration mapping of $f(x)$ that will give zero calibration error. Furthermore, given only access to $f(x)$, predicting $\mathbb{E}[y|f(x)]$ emerges as the Bayesian optimal classifier. 

Nevertheless, a direct computation of $\mathbb{E}[y|f(x)]$ poses challenges since it depends on the unknown data generating distribution. To address this, we adopt a kernel density estimator (KDE) that is proven to be consistent \cite{KDE-XEPopordanoskaNeurIPS2022}:
\begin{equation} \label{eq:kde}
     \mathbb{E}[y|f(x)] \approx \widehat{\mathbb{E}[y|f(x)]} = \frac{\sum_{i=1}^{n}k_{\text{kernel}}(f(x), f(x_i))y_i}{\sum_{i=1}^{n}k_{\text{kernel}}(f(x), f(x_i))}.
\end{equation}
For binary classification, $k_{\text{kernel}}$ is a Beta kernel:
\begin{equation}
    k_{\text{Beta}}(f(x_j),f(x_i)) = f(x_j)^{\alpha_i-1} (1-f(x_j))^{\beta_i-1} \frac{\Gamma(\alpha_i + \beta_i)}{\Gamma(\alpha_i)\Gamma(\beta_i)},
\end{equation}
and for the multiclass setting, $k_{\text{kernel}}$ is a Dirichlet kernel:
\begin{equation}
    k_{\text{Dir}}(f(x_j),f(x_i)) = \frac{\Gamma(\sum_{k=1}^K \alpha_{ik})}{\prod_{k=1}^K \Gamma(\alpha_{ik})} \prod_{k=1}^K f(x_j)_{k}^{\alpha_{ik}-1}
\end{equation}
with $\Gamma(\cdot)$ the gamma function, $\alpha_{i} = \frac{f(x_i)}{h} + 1$ and $\beta_i=\frac{1-f(x_i)}{h} + 1$, where $h \in \mathbb{R}_{>0}$ is a bandwidth parameter that controls smoothness of the kernel density estimation. Like the temperature in temperature scaling, the estimation becomes smoother when we increase $h$.

The computational complexity of KDE is $O(n)$ for a single pixel, leading to an overall complexity of $O(n^2)$. During gradient descent training, KDE is estimated using a mini-batch. Nonetheless, in semantic segmentation, $n=B\times H \times W$, where $B$ is the batch size, $H$ and $W$ denote the height and width of the image, respectively. Given that $n$ can be significantly large, it becomes necessary to manage the computational cost. To this end, for each training data batch, we randomly select $n_{\text{key}}$ key points such that $n_{\text{key}} \ll n$. Consequently, for each pixel $x$, we only perform the kernel computation with respect to these key points. As a result, the complexity is considerably reduced from $O(n^2)$ to $O(n\times n_{\text{key}})$, incurring only a marginal additional computational cost.

Another inherent challenge in semantic segmentation is that labels are usually highly unbalanced. When sampling key points, there is a risk that certain classes might be overlooked. Consequently, during kernel computation, predictions corresponding to these missing classes will become 0. To counter this, if a training data batch contains $n_{\text{unique}}$ unique classes, we sample $\lceil \frac{n_{\text{key}}}{n_{\text{unique}}} \rceil$ pixels for each unique class. Moreover, since medical images predominantly consist of background pixels, we empirically find that it suffices to apply KDE only to misclassified pixels as well as pixels that are near the boundary.

Through the kernel computation process, the information from the selected key points is broadcasted to all pixels. We leverage both the labels and predictions of these key points to refine the predictions of each individual pixel. In Table \ref{tb:kdloss}, we conduct an ablation study to assess the impact of KDE. Meanwhile, Table \ref{tb:bandwidth} illustrates the influence of bandwidth on LiTS and KiTS using a UNet-ResNet18 student. We notice the performance of KDE exhibits noticeable sensitivity to the bandwidth parameter, but the optimal value is consistent across the two datasets.

\end{document}